\documentclass[12pt]{article}
\usepackage{amsmath, amssymb, amsfonts}
\usepackage{amsthm}
\usepackage[utf8]{inputenc}
\usepackage{csquotes}
\usepackage{geometry}
\usepackage{adjustbox}
\usepackage{graphicx}
\usepackage{verbatim}
\usepackage{array}
\usepackage{enumitem}
\usepackage{booktabs}
\usepackage{hyperref}
\usepackage{url}
\usepackage[capitalize]{cleveref}
\usepackage{tikz}

\title{Smooth Approximations of the Rounding Function}
\author{Stanislav Semenov \\
\href{mailto:stas.semenov@gmail.com}{stas.semenov@gmail.com} \\
\href{https://orcid.org/0000-0002-5891-8119}{ORCID: 0000-0002-5891-8119}}
\date{April 26, 2025}

\theoremstyle{definition}
\newtheorem{definition}{Definition}[section]

\theoremstyle{plain}

\newtheorem{proposition}[definition]{Proposition}

\theoremstyle{remark}

\begin{document}

\maketitle

\begin{abstract}
We propose novel smooth approximations to the classical rounding function, suitable for differentiable optimization and machine learning applications. Our constructions are based on two approaches: (1) localized sigmoid window functions centered at each integer, and (2) normalized weighted sums of sigmoid derivatives representing local densities. The first method approximates the step-like behavior of rounding through differences of shifted sigmoids, while the second method achieves smooth interpolation between integers via density-based weighting. Both methods converge pointwise to the classical rounding function as the sharpness parameter \( k \to \infty \), and allow controlled trade-offs between smoothness and approximation accuracy. We demonstrate that by restricting the summation to a small set of nearest integers, the computational cost remains low without sacrificing precision. These constructions provide fully differentiable alternatives to hard rounding, which are valuable in contexts where gradient-based methods are essential.
\end{abstract}

\subsection*{Mathematics Subject Classification}
03F60 (Constructive and recursive analysis), 26E40 (Constructive analysis)

\subsection*{ACM Classification}
F.4.1 Mathematical Logic, F.1.1 Models of Computation

\section{Introduction}

Rounding real-valued data to the nearest integer is a fundamental operation that appears across a wide range of fields, including numerical analysis, signal processing, computer graphics, machine learning, and optimization. In many practical applications, rounding is required to enforce discrete constraints, quantize continuous signals, or map continuous outputs onto a finite set of discrete choices.

Despite its simplicity, the classical rounding function is inherently discontinuous, introducing non-differentiabilities that pose significant challenges in modern computational frameworks. In particular, tasks involving gradient-based optimization --- such as training neural networks, solving relaxed combinatorial problems, or implementing differentiable programming pipelines --- require smooth approximations of non-smooth operations~\cite{Goodfellow2016,Bishop2006}. The inability to backpropagate through hard rounding severely limits the applicability of classical methods in these contexts.

Traditional approaches to mitigate this issue include the use of stochastic rounding, softmax relaxations, or smooth approximations based on the sigmoid or hyperbolic tangent functions~\cite{Rudin1976}. For example, a common trick is to apply a scaled sigmoid function to approximate the floor or ceiling operations, effectively introducing a soft threshold. Similarly, soft quantization techniques replace hard discretization with continuous interpolations. However, many of these methods lack precise control over the sharpness of the transition, struggle to localize the approximation near integers, or introduce biases that degrade performance in precision-sensitive tasks.

In this work, we propose two systematic constructions of smooth rounding functions that address these limitations. The first is based on localized differences of sigmoid functions centered around each integer, producing smoothed step-like transitions. The second is based on normalized weighted sums of sigmoid derivatives, interpreting the rounding as a continuous interpolation between neighboring integers according to local density contributions. Both methods are fully differentiable, allow explicit control over the transition sharpness via a parameter \( k \), and converge to classical rounding behavior as \( k \to \infty \). Furthermore, we show that by restricting computation to a small number of nearest integers, we can retain computational efficiency without sacrificing approximation quality.

Our proposed smooth rounding schemes are particularly well-suited for applications that demand end-to-end differentiability, including differentiable rendering, discrete latent variable modeling, smooth quantization, differentiable combinatorial optimization, and relaxed integer programming.

\subsection*{Summary of Contributions}

In this paper, we make the following contributions:
\begin{itemize}
    \item We propose a fully differentiable smooth approximation to the classical rounding function, constructed via localized differences of sigmoid functions without relying on hard operations like \(\mathrm{floor}(x)\) or \(\mathrm{ceil}(x)\).

    \item We introduce an alternative smooth rounding construction via normalized weighted sums of sigmoid derivatives, enabling continuous interpolation between neighboring integers based on local density contributions.
    
    \item We show that both methods converge pointwise to the classical rounding function as the sharpness parameter \( k \to \infty \), providing tunable control over the trade-off between smoothness and approximation accuracy.
    
    \item We demonstrate that computational efficiency is achieved by restricting attention to a small number of nearest integers without significant loss of precision, making the methods practical for large-scale and real-time applications.
    
    \item We highlight potential applications in differentiable programming, smooth quantization, discrete optimization relaxations, and neural network architectures requiring end-to-end differentiability.
\end{itemize}

\section{Rounding Functions and Smooth Approximations}

In this section, we describe the classical rounding function and two smooth approximations proposed in this work. We highlight their construction principles, mathematical properties, and computational characteristics.

\subsection{Classical Rounding}

The classical rounding operation maps each real number \( x \in \mathbb{R} \) to its nearest integer:
\[
\mathrm{round}(x) := \text{nearest integer to } x.
\]
This mapping exhibits a sharp discontinuous jump at every half-integer point \( x = n + 0.5 \), where \( n \in \mathbb{Z} \).  
Although simple and exact, the classical rounding function is non-differentiable, making it unsuitable for applications that require smoothness and gradient information.

\textbf{Key properties:}
\begin{itemize}
    \item Sharp discontinuities at half-integers.
    \item Exact integer mapping.
    \item Non-differentiable everywhere.
    \item Time complexity: \( O(N) \).
    \item Space complexity: \( O(N) \).
\end{itemize}

\subsection{Smooth Rounding via Sigmoid Differences}

To obtain a differentiable approximation of rounding, we first introduce a method based on localized sigmoid windows.  
For each integer \( n \), we define a soft window function by taking the difference of two shifted sigmoids:
\[
\mathrm{round}_\sigma(x, k) := \sum_{n} n \left[ \sigma\left(k(x - (n-0.5))\right) - \sigma\left(k(x - (n+0.5))\right) \right],
\]
where \( \sigma(z) = (1 + e^{-z})^{-1} \) is the standard sigmoid function, and \( k > 0 \) controls the sharpness of the transition.

In this construction, each \( n \) contributes significantly only when \( x \) is close to \( n \), within approximately one unit.

\textbf{Key properties:}
\begin{itemize}
    \item Smooth approximation of a step function around each integer.
    \item Sharp but smoothed transitions near half-integers.
    \item Converges pointwise to \(\mathrm{round}(x)\) as \( k \to \infty \).
    \item Time complexity: \( O(NM) \) with \( M \sim 5 \) nearest neighbors.
    \item Space complexity: \( O(N) \).
\end{itemize}

\subsection{Normalized Smooth Rounding via Sigmoid Derivatives}

Alternatively, we propose a smoother and more globally consistent approximation by considering normalized weighted sums of local sigmoid derivatives.  
The local contribution of each integer \( n \) is defined by:
\[
\rho_n(x) := k \, \sigma(k(x-n))(1-\sigma(k(x-n))),
\]
which forms a bell-shaped density around \( n \).  
The smooth rounding function is then given by:
\[
\mathrm{round}_{\text{norm}}(x, k) := \frac{ \sum_n n \rho_n(x) }{ \sum_n \rho_n(x) }.
\]

Here, each integer \( n \) contributes proportionally to its local density, and the final result is a weighted average of nearby integers.  
This method naturally ensures that the function is fully smooth and that the weights sum to one.

\textbf{Key properties:}
\begin{itemize}
    \item Fully smooth and differentiable.
    \item Continuous interpolation between neighboring integers.
    \item Converges pointwise to \(\mathrm{round}(x)\) as \( k \to \infty \).
    \item Time complexity: \( O(NM) \) with \( M \sim 5 \) nearest neighbors.
    \item Space complexity: \( O(N) \).
\end{itemize}

\subsection{Comparison of Methods}

The following table summarizes the principal characteristics of the classical and smooth rounding functions:

\begin{table}[ht]
\centering
\resizebox{\textwidth}{!}{
\begin{tabular}{|c|c|c|c|}
\hline
\textbf{Property} & \textbf{Classical Round} & \textbf{Sigmoid Differences} & \textbf{Normalized Derivative} \\
\hline
Continuity & Discontinuous & Smooth (piecewise) & Fully smooth \\
\hline
Sharpness at Half-Integers & Infinite jump & Smoothed jump & Smooth transition \\
\hline
Mechanism & Hard rounding & Window via sigmoid differences & Weighted average via densities \\
\hline
Limit as \( k \to \infty \) & Exact & Exact & Exact \\
\hline
Typical Neighbors Needed & 0 (instant) & 5 neighbors & 5 neighbors \\
\hline
Time Complexity & \( O(N) \) & \( O(NM) \) & \( O(NM) \) \\
\hline
Space Complexity & \( O(N) \) & \( O(N) \) & \( O(N) \) \\
\hline
\end{tabular}
}
\caption{Comparison of classical rounding, sigmoid-difference smoothing, and normalized derivative smoothing.}
\label{tab:rounding_comparison}
\end{table}

Both smooth approximations maintain full differentiability and converge to the classical rounding function as the sharpness parameter increases.
However, they differ in how transitions between integers are handled: the sigmoid-difference method produces sharp but smoothed boundaries near half-integer points, whereas the normalized-derivative method achieves a fully continuous and gradual interpolation across the entire real line.

\section{Analysis of Smooth Rounding Functions}

In this section, we analyze the key mathematical properties of the proposed smooth rounding functions: convergence to classical rounding, smoothness, and computational efficiency.

\subsection{Pointwise Convergence as \texorpdfstring{\( k \to \infty \)}{k → ∞}}

Both smooth approximations are parameterized by a sharpness parameter \( k > 0 \) that controls the steepness of transitions.  
We show that as \( k \to \infty \), both functions converge pointwise to the classical rounding function.

\begin{proposition}
Let \( x \in \mathbb{R} \) be arbitrary. Then:
\[
\lim_{k \to \infty} \mathrm{round}_\sigma(x,k) = \mathrm{round}(x),
\quad
\lim_{k \to \infty} \mathrm{round}_{\text{norm}}(x,k) = \mathrm{round}(x).
\]
\end{proposition}

\begin{proof}[Sketch of Proof]
For the sigmoid-difference method, as \( k \to \infty \), the sigmoids \(\sigma(k(x-(n-0.5)))\) and \(\sigma(k(x-(n+0.5)))\) become step functions centered at \(n-0.5\) and \(n+0.5\) respectively, producing a sharp window around each integer.  
Thus, the sum selects the nearest integer.

For the normalized-derivative method, the densities \(\rho_n(x)\) concentrate sharply at the nearest integer as \(k \to \infty\), and the weighted average converges to the corresponding integer value.
\end{proof}

\subsection{Smoothness and Differentiability}

Both smooth rounding functions are infinitely differentiable for any finite \(k\), since they are constructed from compositions and sums of sigmoid functions, which are themselves \( C^\infty \).

Moreover, the gradients with respect to \( x \) can be computed explicitly:

- For the sigmoid-difference method:
\[
\frac{d}{dx} \mathrm{round}_\sigma(x,k) = k \sum_{n} n \left[ \sigma'\left(k(x-(n-0.5))\right) - \sigma'\left(k(x-(n+0.5))\right) \right].
\]
- For the normalized-derivative method, differentiation involves applying the quotient rule:
\[
\frac{d}{dx} \mathrm{round}_{\text{norm}}(x,k) = \frac{ \sum_n n \rho'_n(x) \cdot \sum_n \rho_n(x) - \sum_n n \rho_n(x) \cdot \sum_n \rho'_n(x) }{ \left( \sum_n \rho_n(x) \right)^2 },
\]
where \( \rho'_n(x) = k^2 (1-2\sigma(k(x-n))) \sigma(k(x-n))(1-\sigma(k(x-n))) \).

These expressions confirm that the functions are fully differentiable and suitable for gradient-based methods.

\subsection{Computational Considerations}

Although the theoretical definitions involve infinite sums over \( n \in \mathbb{Z} \), in practice, the contribution of distant integers decays exponentially fast with \( k \).  
Specifically, for large \( k \), the local densities \(\rho_n(x)\) or the differences of sigmoids become negligible outside a small neighborhood of \( x \).

Thus, it is sufficient to restrict the summations to a fixed small number of nearest integers, typically \(5\) or fewer, without significant loss of accuracy.

This ensures that the computational complexity per evaluation remains \( O(M) \) with \( M \sim 5 \), leading to an overall complexity \( O(NM) \) for \( N \) input values.

\subsection{Summary of Analytical Properties}

\begin{itemize}
    \item Both methods converge pointwise to classical rounding as \(k \to \infty\).
    \item Both methods are \( C^\infty \) smooth for any finite \(k\).
    \item Gradients can be computed explicitly, enabling backpropagation.
    \item Efficient evaluation is achieved by localizing computation to a small set of neighboring integers.
\end{itemize}

\section*{Conclusion}

We have proposed two smooth approximations to the classical rounding function, based on localized differences of sigmoid functions and normalized weighted sums of sigmoid derivatives. Both constructions achieve fully differentiable behavior across the real line and converge pointwise to classical rounding as the sharpness parameter increases. 

A key advantage of our approach is that it entirely avoids the use of non-differentiable operations such as \(\mathrm{floor}(x)\) or \(\mathrm{ceil}(x)\), relying solely on smooth analytic functions. This makes the methods particularly suitable for applications requiring end-to-end differentiability, such as differentiable programming, optimization, and machine learning.

\newpage
\appendix

\section{Appendix: Python Implementation}

In this appendix, we provide simple Python code that implements the proposed smooth rounding functions and compares them to the classical rounding.

\begin{verbatim}
import numpy as np
import matplotlib.pyplot as plt

def sigma(z):
    return 1 / (1 + np.exp(-z))

def smooth_round_precise(x, k):
    result = np.zeros_like(x)
    for i in range(len(x)):
        xi = x[i]
        n0 = np.floor(xi)
        neighbors = np.array([n0-1, n0, n0+1, n0+2])
        rhos = (
            k * sigma(k * (xi - neighbors)) *
            (1 - sigma(k * (xi - neighbors)))
        )
        numerator = np.sum(neighbors * rhos)
        denominator = np.sum(rhos) + 1e-8
        result[i] = numerator / denominator
    return result

def classical_round(x):
    return np.round(x)

def round_sigma_original(x, k):
    result = np.zeros_like(x)
    for i in range(len(x)):
        xi = x[i]
        n_range = np.arange(np.floor(xi) - 2, np.floor(xi) + 3)
        for n in n_range:
            result[i] += n * (
                sigma(k * (xi - (n - 0.5)))
                - sigma(k * (xi - (n + 0.5)))
            )
    return result

def round_sigma_derivative_normalized(x, k):
    numerator = np.zeros_like(x)
    denominator = np.zeros_like(x)
    for i in range(len(x)):
        xi = x[i]
        n_range = np.arange(np.floor(xi) - 2, np.floor(xi) + 3)
        for n in n_range:
            rho = k * sigma(k * (xi - n)) * (1 - sigma(k * (xi - n)))
            numerator[i] += n * rho
            denominator[i] += rho
    denominator = np.maximum(denominator, 1e-8)
    return numerator / denominator

k = 10
x = np.linspace(-1, 1, 1000)

y_classical = classical_round(x)
y_original = round_sigma_original(x, k)
y_normalized = round_sigma_derivative_normalized(x, k)

plt.figure(figsize=(14, 8))
plt.plot(
    x, y_classical,
    label='Classical round(x)',
    linestyle='-',
    linewidth=2
)
plt.plot(
    x, y_original,
    label='Original round_sigma(x,k)',
    linestyle='--',
    linewidth=2
)
plt.plot(
    x, y_normalized,
    label='Normalized derivative approximation',
    linestyle='-.',
    linewidth=2
)
plt.title(f'Comparison of Rounding Functions (k={k})')
plt.xlabel('x')
plt.ylabel('Value')
plt.grid(True)
plt.legend()
plt.show()
\end{verbatim}

As illustrated in Figure~\ref{fig:smooth_rounding}, the plots generated by this code show the behavior of the classical rounding function, the smooth rounding via sigmoid differences, and the smooth rounding via normalized sigmoid derivatives.

\bigskip

Note: To maintain numerical stability and prevent division by very small numbers, a small constant (\(10^{-8}\)) is added to the denominator in the normalized method.

\begin{figure}[ht]
\centering
\includegraphics[width=0.8\textwidth]{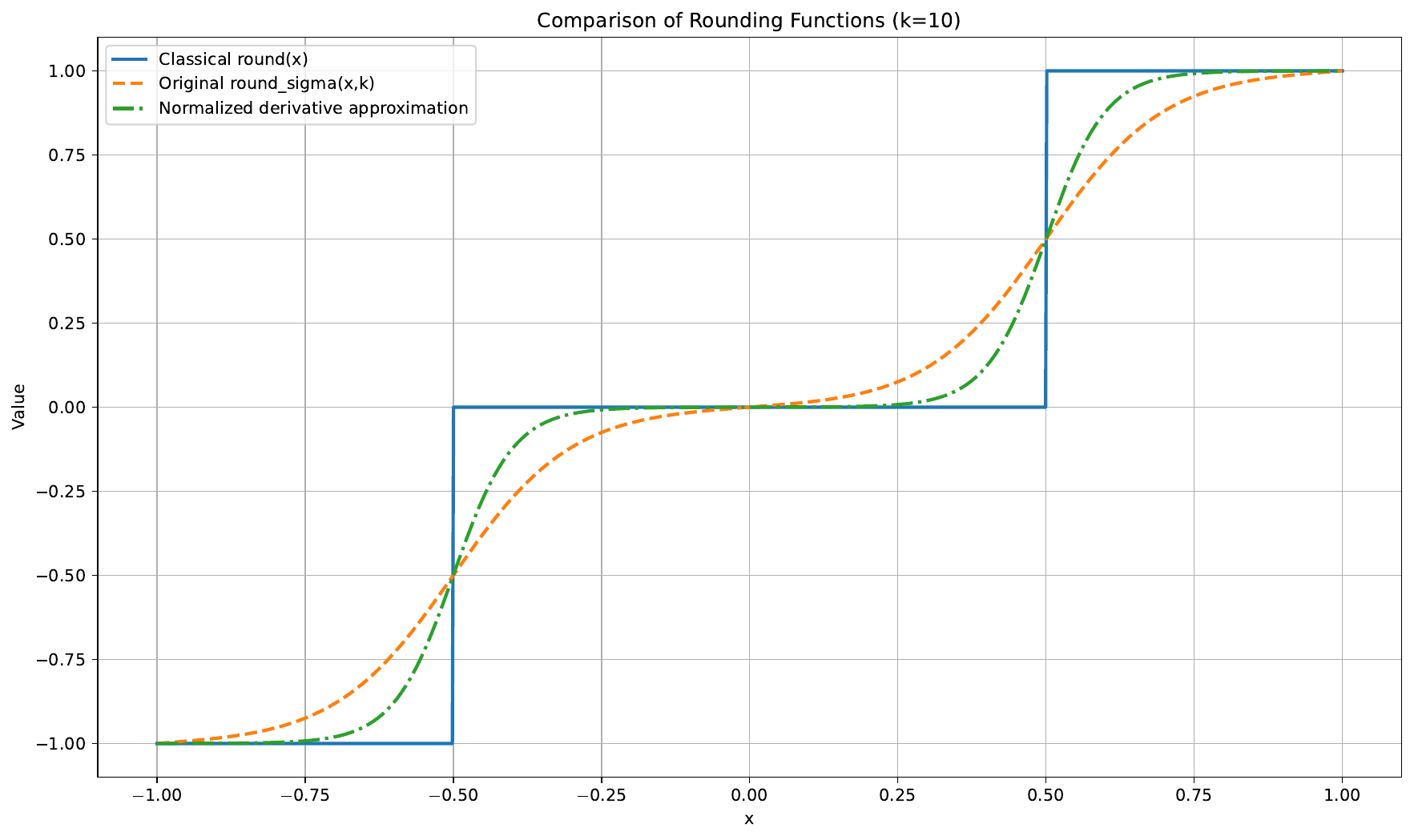}
\caption{Comparison of classical rounding and smooth approximations (\(k=10\)).}
\label{fig:smooth_rounding}
\end{figure}

\end{document}